\documentclass[twoside]{article}

\usepackage[accepted]{aistats2020}
\usepackage[round]{natbib}

\usepackage{booktabs}

\usepackage{algorithm}
\usepackage[noend]{algpseudocode}
\usepackage{color}
\usepackage{xcolor}

\usepackage{algpseudocode}
\usepackage{amsfonts}
\usepackage{amsthm}
\usepackage{graphicx}
\def\algbackskip{\hskip-\ALG@thistlm}
\usepackage{enumitem}
 
\newtheorem{theorem}{Theorem}[section]

\newtheorem{lemma}[theorem]{Lemma}
 
\begin{document}

%

%

\twocolumn[

\aistatstitle{Orthogonal Gradient Descent for Continual Learning}

\aistatsauthor{ Mehrdad Farajtabar
\And Navid Azizan$^{1}$ \And Alex Mott \And Ang Li }

\aistatsaddress{ DeepMind \And  CalTech \And DeepMind \And DeepMind } ]

\begin{abstract}
\vspace{-2mm}
Neural networks are achieving state of the art and sometimes super-human performance on learning tasks across a variety of domains. Whenever these problems require learning in a continual or sequential manner, however, neural networks suffer from the problem of \emph{catastrophic forgetting}; they forget how to solve previous tasks after being trained on a new task, despite having the essential capacity to solve both tasks if they were trained on both simultaneously. 
In this paper, we propose to address this issue from a parameter space perspective and study an approach to restrict the direction of the gradient updates to avoid forgetting previously-learned data. We present the Orthogonal Gradient Descent (OGD) method, which accomplishes this goal by projecting the gradients from new tasks onto a subspace in which the neural network output on previous task does not change and the projected gradient is still in a useful direction for learning the new task. Our approach utilizes the high capacity of a neural network more efficiently and does not require storing the previously learned data that might raise privacy concerns. Experiments on common benchmarks reveal the effectiveness of the proposed OGD method.
\vspace{-2mm}
\end{abstract}

\vspace{-2mm}
\section{Introduction}
\vspace{-2mm}
One critical component of intelligence is the ability to learn \emph{continuously}, when new information is constantly available but previously presented information is unavailable to retrieve. Despite their ubiquity in the real world, these problems have posed a long-standing challenge to artificial intelligence~\citep{thrun1995lifelong,hassabis2017neuroscience}. 

A typical neural network training procedure over a sequence of different tasks usually results in degraded performance on previously trained tasks if the model could not revisit the data of previous tasks. This phenomenon is called \emph{catastrophic forgetting}~\citep{mccloskey1989catastrophic,ratcliff1990connectionist,french1999catastrophic}. Ideally, an intelligent agent should be able to learn consecutive tasks without degrading its performance on those already learned.
With the deep learning renaissance~\citep{krizhevsky2012imagenet,hinton2006fast,simonyan2014very} this problem has been revived  ~\citep{srivastava2013compete,goodfellow2013empirical} with many follow-up studies~\citep{parisi2019continual}.

One probable reason for this phenomenon is that neural networks are usually trained by Stochastic Gradient Descent (SGD)---or its variants---where the optimizers produce gradients that are oblivious to past knowledge. These optimizers, by design, produce gradients that are purely a function of the current minibatch (or some smoothed average of a short window of them). This is a desirable feature when the training data is iid, but is not desirable when the training distribution shifts over time. 
In this paper, we present a system where the gradients produced on a training minibatch can avoid interfering with gradients produced on previous tasks.

The core idea of our approach, Orthogonal Gradient Descent (OGD), is to preserve the previously acquired knowledge by maintaining a space consisting of the gradient directions of the neural network predictions on previous tasks. Any update orthogonal to this gradient space change the output of the network minimally. When training on a new task, OGD projects the loss gradients of new samples perpendicular to this gradient space before applying back-propagation. 
Empirical results demonstrate that the proposed method efficiently utilizes the high capacity of the (often over-parameterized) neural network to learn the new data while minimizing the interference with the previously acquired knowledge. Experiments on three common continual learning benchmarks substantiate that OGD achieves state-of-the-art performance without the need to store the historical data.

\section{Preliminaries}
\label{sec:prelim}
Consider a \textit{continual learning} setting in which tasks $\{T_1, T_2, T_3, \ldots\}$ arrive sequentially. When a model is being trained on task $T_k$, any data from previous tasks $\{T_t~|~t<k\}$ is inaccessible.
Each data point $(x,y)\in T_k$ is a pair consists of input $x\in \mathbb{R}^d$ and a label $y$. For a $c$-class classification, $y$ is a $c$-dimensional one hot vector.
The prediction of the model on input $x$ is denoted by $f(x; w)$, where $w\in\mathbb{R}^p$ are parameters (weights) of the model (neural network). For classification problems, $f(x; w) \in \mathbb{R}^c$ where $f_j(x; w)$ is the $j$-th logit associated to $j$-th class.

The total loss on the training set (empirical risk) for task $t$ is denoted by
\begin{equation}L_t(w)=\sum_{(x,y)\in T_t} L_{(x,y)}(w),\end{equation}
where the per-example loss is defined as
\begin{equation}L_{(x,y)}(w) = \ell(y,f(x; w)),\end{equation}
and $\ell(\cdot,\cdot)$ is a differentiable non-negative loss function. 
For classification problems, a softmax cross entropy loss is commonly used, \textit{i.e.},
\begin{equation}
    \ell(y, f(x;w)) = - \sum_{j=1}^c y_j \log a_j,
\end{equation} where $a_j = \exp f_j(x;w)/\sum_k \exp f_k(x;w)$ is the $j$-th softmax output.

Two objects that frequently appear throughout the development of our method are the gradient of the \emph{loss}, $\nabla L_{(x,y)}(w)\in\mathbb{R}^p$, and the gradient of the \emph{model}, $\nabla f(x;w)\in\mathbb{R}^{p \times c}$, which are both with respect to $w$, and it is critically important to distinguish the two. In fact, the gradient of the loss, using the chain rule, can be expressed as
\begin{equation}
\label{eq:loss-vs-prediction}
    \nabla L_{(x,y)}(w) = \nabla f(x;w) \ell'(y,f(x;w)),
\end{equation}
where $\ell'(\cdot,\cdot) \in \mathbb{R}^c$ denotes the derivative of $\ell(\cdot,\cdot)$ with respect to its second argument, and $\nabla f(x;w)$ is the gradient of the model $f$ with respect to its second argument (\emph{i.e.}, the parameters). For the classification problem with cross entropy softmax loss, we have
\begin{equation}
\nabla f(x;w) = [\nabla f_1(x;w); \ldots; \nabla f_c(x;w), ],
\end{equation}
and the derivative of the loss becomes 
\begin{equation}
\ell'(y, f(x;w)) = [a_1-y_1, \ldots, a_c-y_c]^{\top},
\end{equation}
where, $ \nabla f_j(x;w) \in \mathbb{R}^{c} $ is  the gradient of the $j$-th logit with respect to parameters.
%

\begin{figure}[t]
    \centering
    \includegraphics[width=0.48\textwidth]{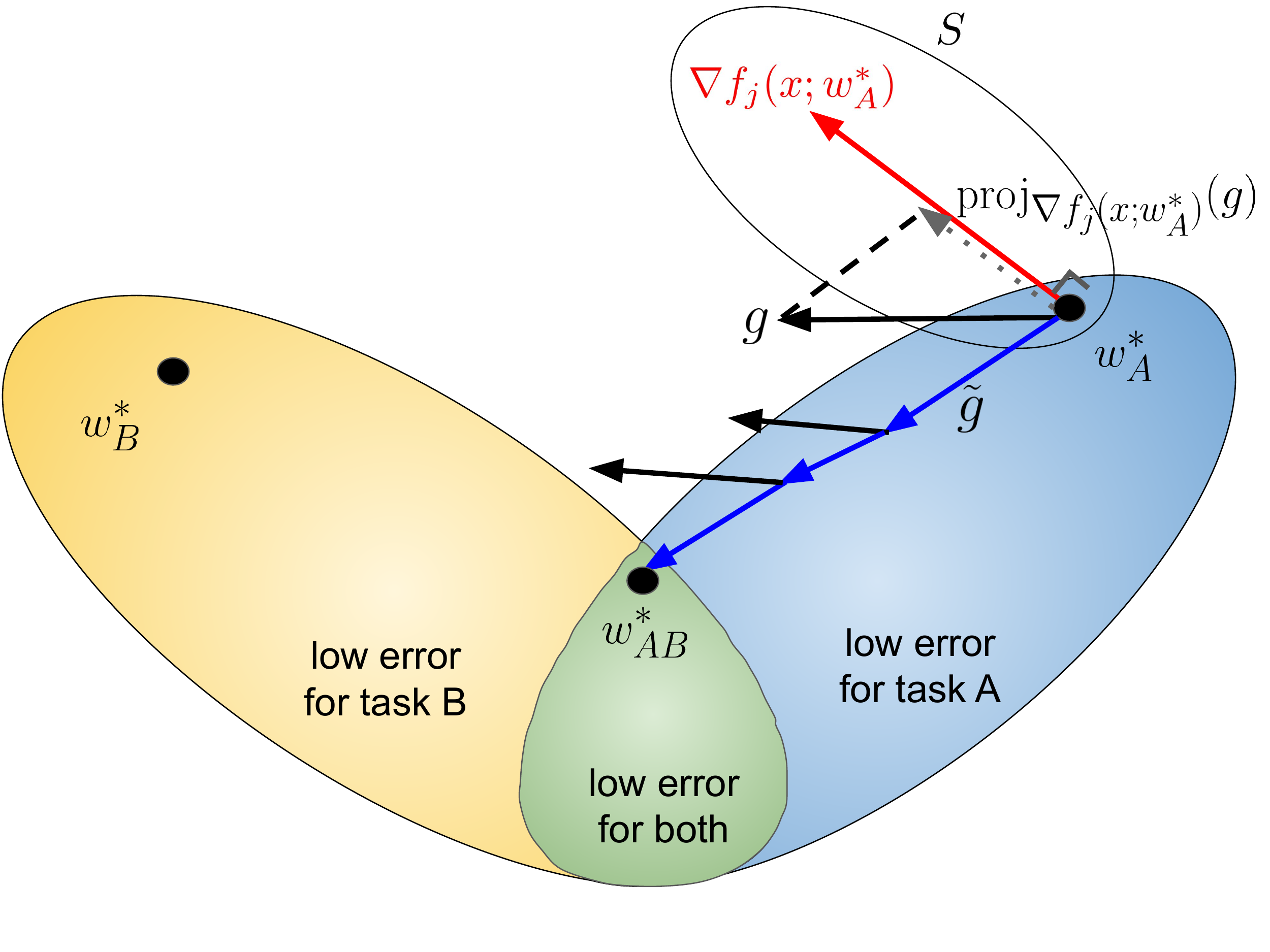}
    \vspace{-4mm}
    \caption{\small An illustration of how Orthogonal Gradient Descent corrects the directions of the gradients. $g$ is the original gradient computed for task B and $\tilde{g}$ is the projection of $g$ onto the orthogonal space \textit{w.r.t} the gradient $\nabla f_j(x;w_A^*)$ computed at task A. Moving within this (blue) space allows the model parameters to get closer to the low error (green) region for both tasks.}
    \label{fig:overview}
\end{figure}

\section{Orthogonal Gradient Descent}
\label{sec:proposed_ogd}
Catastrophic forgetting happens in neural networks when the gradient updates with respect to a new task are applied to the model without considering previous tasks. 
We propose the Orthogonal Gradient Descent method for mitigating this problem, which is based on modifying the direction of the updates to account for \emph{important directions} of previous tasks. Figure \ref{fig:overview} shows an illustration of the core idea of OGD which constrains the parameters to move within the orthogonal space to the gradients of previous tasks.

Suppose that the model has been trained on the task A in the usual way until convergence to a parameter vector $w^*_A$ so that the training loss/error is small or zero, and consider the effect of a small update to $W^*_A$.
In the high-dimensional parameter space of the model, there could be update directions causing large changes in the predictions from $x \in T_A$, while there also exist updates that minimally affect such predictions. In particular, moving locally along the direction of $\pm \nabla f_j(x;w)$ leads to the biggest change in model prediction $f_j(x;w)$ given any sample $x$, while moving orthogonal to $\nabla f_j(x;w)$ leads to the least change (or no change, locally) to the prediction of $x$. 
Supposing task A has $n_A$ data points in the stochastic gradient descent setting, there will be $n_A \times c$ gradient directions $\nabla f_j(x;w)$. In order to guarantee the least change to the predictions on task A, for an update towards task B, the update has to be orthogonal to the $n_A \times c$ directions $\{\nabla f_j(x;w)\}_{x \in T_A, j=1\ldots c}$.

Denoting the gradient of the loss for task B (which can be either stochastic, batch, or full) by $g$, we propose to ``orthogonalize'' it in a way that the new direction $\tilde{g}$ satisfies the above requirement, i.e.,
\begin{equation}
\tilde{g} \perp \nabla f_j(x;w),\quad \forall x \in T_A,j=1\ldots c~.
\end{equation}
In this case, moving along direction $\tilde{g}$ makes the least change to the neural network predictions on the previous task. As a result, we utilize the high capacity of neural networks more effectively. 
We know that a neural network is part of a high dimensional parameter space (larger than or comparable to the number of data points), so there always exist a direction that conforms to the orthogonality condition.

In continual leaning, while processing $T_B$, one does not have access to $T_A$ anymore to compute $\nabla f_j(x;{w})$ at the current parameter $w$. This means that, as an inherent limitation, we are unable to compute the exact directions that will produce least changes on Task A during training on Task B.
To tackle this issue note that the neural networks are often overparameterized, which implies that in a close vicinity of the optimum parameter for task A, there lies optimum parameters for both tasks A and B~\citep{azizan2018stochastic,li2018learning,allen2018convergence,azizan2019stochastic}.  For any parameter $w$ in that neighborhood, we can basically approximate $\nabla f(x;{w}) \approx \nabla f(x;{w^*_A})$ for all $x \in T_A$.
Therefore, we can use $\nabla f(x;{w^*_A})$ as proxy and satisfy
\begin{equation}
   \tilde{g} \perp  \nabla f_j(x; w^*_A), \quad  \forall x \in T_A, j=1\ldots c,
\end{equation}
for all (batch) loss gradients $\tilde{g}$ of task B. 
One can compute and store $\nabla f(x;{w^*_A})$ for all $x\in T_A$ when training on task A is done and task B is introduced. 

In practice, one does not need all $n_A \times c$ directions $\{\nabla f_j(x;w)\}_{x \in T_A, j=1\ldots c}$ to preserve the previously learned information. For example, per sample $x$, we can compute the gradient with respect to the average of the logits rather than use the individual logits themselves. We call this OGD-AVE in contrast to OGD-ALL.
Another alternative is to select the single logit corresponding to the ground truth label. For data point $x$ coming from $k$-th class ($y_k=1$), we try to only keep $\nabla f_k(x; w)$ invariant. 
This alternative referred to as OGD-GTL. Both OGD-GTL and OGD-AVE reduce the storage size by a factor of $c$. 
We use OGD-GTL in all of our following experiments and also empirically observe that OGD-GTL slightly outperforms the other variants of OGD (\emph{c.f} appendix ~\ref{sec:app_variants}). To further control the amount of memory required for this process, we store only a subset of gradients from each task in our experiments (200 for the Mnist experiments). While this potentially misses some information, we find in practice that it is sufficient and that increasing the collection size beyond this provides diminishing returns.
One downside of this method (similar to other state-of-the-art methods such as~\citep{chaudhry2018efficient, lopez2017gradient,riemer2018learning}) is that the required storage size grows with the number of tasks. An interesting extension to our method is to dynamically remove less significant directions from set $S$ or perform principal component analysis on the gradient space, which are left for future work.

\begin{algorithm}[t]
  \caption{ Orthogonal Gradients Descent}\label{alg:ogd}
      \hspace*{\algorithmicindent} \textbf{Input} Task sequence $T_1, T_2, T_3, \ldots$ learning rate $\eta$\\
    \hspace*{\algorithmicindent} \textbf{Output} The optimal parameter $w$.
  \begin{algorithmic}[1]
      \State {\bf Initialize} $S\gets\{\}$;  $w\gets w_0$
      \For{Task ID $k = 1,2,3,\dots$}
        \Repeat 
            \State $g\gets$ Stochastic/Batch Gradient for $T_k$ at $w$
            \State $\tilde{g}=g-\sum_{v\in S} \mathrm{proj}_v (g)$
            \State $w \gets w- \eta \tilde{g}$
        \Until{ convergence}
        \For{$(x,y) \in T_t$ and $k\in [1, c]\text{ s.t. }y_k=1$} 
            \State $u\gets \nabla f_k(x;w)-\sum_{v\in S}\mathrm{proj}_v(\nabla f_k(x;w))$
            \State $S \gets S\cup \{u\}$
        \EndFor
      \EndFor
  \end{algorithmic}
\end{algorithm}

We now proceed to formally introduce OGD-GTL. Task B's loss gradients should be perpendicular to the \emph{space} of all previous model gradients, namely 
\begin{equation*}
    S = \mathrm{span} \{ \nabla f_k(x,{w^*_A})~|~(x,y)\in T_A \wedge k\in[1,c]\wedge y_k=1\}.
\end{equation*}
We compute the orthogonal basis for $S$ as $\{v_1,v_2,\ldots\}$ using the Gram-Schmidt procedure on all gradients \textit{w.r.t.} samples $(x_i, y_i) \in T_A$ in task A. We iteratively project them to the previously orthogonalized vectors:
\begin{align*}
   &  v_{1} = \nabla f_{k_1}(x_1;{w^*_A})~,  \\
  & v_{i} = \nabla f_{k_i}(x_{i};{w^*_A}) - \sum_{j < i}
   \mathrm{proj}_{v_j}(\nabla f(x_{i};{w^*_A})),
\end{align*}
where, $k_i$ represents the ground-truth index such that $y_{i, k_i}=1$, and $\mathrm{proj}_v(u) = \frac{\langle u, v \rangle}{ \langle v, v \rangle} v$  is the  projection (vector) of $u$ in the direction of $v$. 
Given the orthogonal basis $S = \{ v_{1}, \ldots, v_{n_A} \}$  for the gradient subspace of task A, we modify the original gradients $g$ of  task B to new gradients $\tilde{g}$ orthogonal to $S$, \textit{i.e.},
\begin{equation}
    \tilde{g} = g - \sum_{i=1}^{n_A} \mathrm{proj}_{v_{i}}(g)~. 
\end{equation}
The new direction $-\tilde{g}$ is still a descent direction (\emph{i.e.} $\langle-\tilde{g}, g\rangle \leq 0$) for task B meaning that $\exists~ \epsilon > 0$ such that for any learning rate $0 < \eta < \epsilon$, taking the step $\eta \tilde{g}$ reduces the loss.  
%

\begin{lemma}
\label{lem:descent}
Let $g$ be the gradient of loss function $L(w)$ and $S=\{v_1, \ldots, v_n\}$ is the orthogonal basis. Let $\tilde{g} = g - \sum_i^k {\mathrm{proj}_{v_i}(g)}$. Then, $-\tilde{g}$ is also a descent direction for $L(w)$.
\end{lemma}

\begin{proof}
For a vector $u$ to be a descent direction it should satisfy $\langle u, g \rangle \le 0$. To begin with, we have
\begin{align}
\langle -\tilde{g},\,\, g \rangle & = \langle -\tilde{g}, \tilde{g}+\sum_{i=1}^k {\mathrm{proj}_{v_i}(g)}\rangle\\
&=-\|\tilde{g}\|^2-\langle \tilde{g}, \sum_{i=1}^k {\mathrm{proj}_{v_i}(g)}\rangle~.\label{eq:lemma_a2_1}
\end{align}
Since $\tilde{g}=g-\sum_{i=1}^k {\mathrm{proj}_{v_i}(g)}$ is orthogonal to the space spanned by $S$ and $\sum_{i=1}^k {\mathrm{proj}_{v_i}(g)}$ is a vector spanned by $S$, hence $
    \langle \tilde{g}, \sum_{i=1}^k {\mathrm{proj}_{v_i}(g)}\rangle=0$.
Substituting this into Eq.~\ref{eq:lemma_a2_1}, we have $
    \langle -\tilde{g},\,\, g \rangle=-\|\tilde{g}\|^2\le0$.
Therefore, $-\tilde{g}$ is a descent direction for $L(w)$ while being perpendicular to $S$.
\end{proof}

We can  easily extend the method to handle multiple tasks. Algorithm~\ref{alg:ogd} presents this general case.
In this work, we apply the proposed Orthogonal Gradient Descent (OGD) algorithm to continual learning on consecutive tasks. Its application potentially goes beyond this special case and can be utilized whenever one wants the gradient steps minimally interfere with the previous learned data points and potentially reduce the access or iterations over them.
%

It is worth reiterating the distinction made in Section~\ref{sec:prelim} between using the gradient of the logits---as OGD does---and using the gradient of the loss---as many other methods do, including the A-GEM baseline (\cite{chaudhry2018efficient}) in the next section. As Equation~\eqref{eq:loss-vs-prediction} indicates, the gradient of the loss $\nabla L_{(x,y} (w)$ can be zero or close to zero for the examples that are well fitted ($\ell'(y, f(x;w)) \approx 0$) carrying effectively low information on the previous tasks. In contrast, OGD works directly with the model (through its gradeint $\nabla f(x;w)$) which is the essential information to be preserved.

\section{Experiments}
\label{sec:Experiments}
\vspace{-3mm}
We performed experiments in this section on three continual learning benchmark: \textit{Permuted Mnist}, \textit{Rotated Mnist}, and \textit{Split Mnist}.

\textbf{Baselines.} We considered the following baselines for comparison purposes.
(1) \emph{EWC} \citep{kirkpatrick2017overcoming}: one of the pioneering regularization based methods that uses fisher information diagonals as important weights.
(2) \emph{A-GEM} \citep{chaudhry2018efficient}: using loss gradients of stored previous data in an inequality constrained optimization.
(3) \emph{SGD}: Stochastic Gradient Descent optimizing tasks one after the other. It can be seen as lower bound telling us what happens if we do nothing to explicitly retain information from the previous task(s).
(4) \emph{MTL}: Multi Task Learning baseline using stochastic gradient descent with full access to previous data. In this setting, during task $T_t$, we trained the model on batches containing all data $T_{\leq t}$. This can be considered a sort of upper bound on the performance.

\textbf{Setup.}
We used a consistent training setup for all Mnist experiments so that we can directly compare the effects of the model across tasks and methods.
We always trained each task for 5 epochs. 
The number of epochs was chosen to achieve saturated performance on the first task classification problem. 
The performance numbers do not change substantially when trained for more epochs and, crucially, the relative performance between the different methods is identical with more training epochs. 
We used a batch size of 10 similar to~\cite{chaudhry2018efficient,lopez2017gradient}.
We found that batch size was not a strong factor in the performance, other than in its interplay with the number of epochs. For fewer than 5 epochs, the batch size had a noticeable effect because it significantly changed the number of batch updates. 

Large learning rates do degrade the performance of OGD (the larger the learning rate, the more likely a gradient update violates the locality assumption).  We chose a learning rate of $10^{-3}$, consistent with other studies~\citep{kirkpatrick2017overcoming,chaudhry2018efficient}, and small enough that decreasing it further did not improve the performance.
%
For all experiments the same architecture is used. The network is a three-layer MLP with 100 hidden units in two layers and 10 logit outputs. Every layer except the final one uses ReLU activation. The loss is Softmax cross-entropy, and the optimizer is stochastic gradient descent. This setting is similar to previous works~\citep{chaudhry2018efficient,kirkpatrick2017overcoming}. 
At the end of every task boundary we performed some processing required by the method. For OGD, this means computing the orthogonal gradient directions as described in Section~\ref{sec:proposed_ogd}. For A-GEM, this means storing some examples from the ending task to memory. For EWC, this means freezing the model weights and computing the fisher information. 
Both OGD and A-GEM need an storage. A-GEM for actual data points and OGD for the gradients of the model on previous tasks. We set the storage size for both methods to 200.
Last but not least, in all the experiments the mean and standard deviation of the \emph{test error} on the hold out Mnist test set are demonstrated using 10 independent random runs for 2 and 3 task experiments and 5 independent runs for 5 task experiments.

\begin{figure}[!tb]
    \centering
    \includegraphics[width=0.48\textwidth]{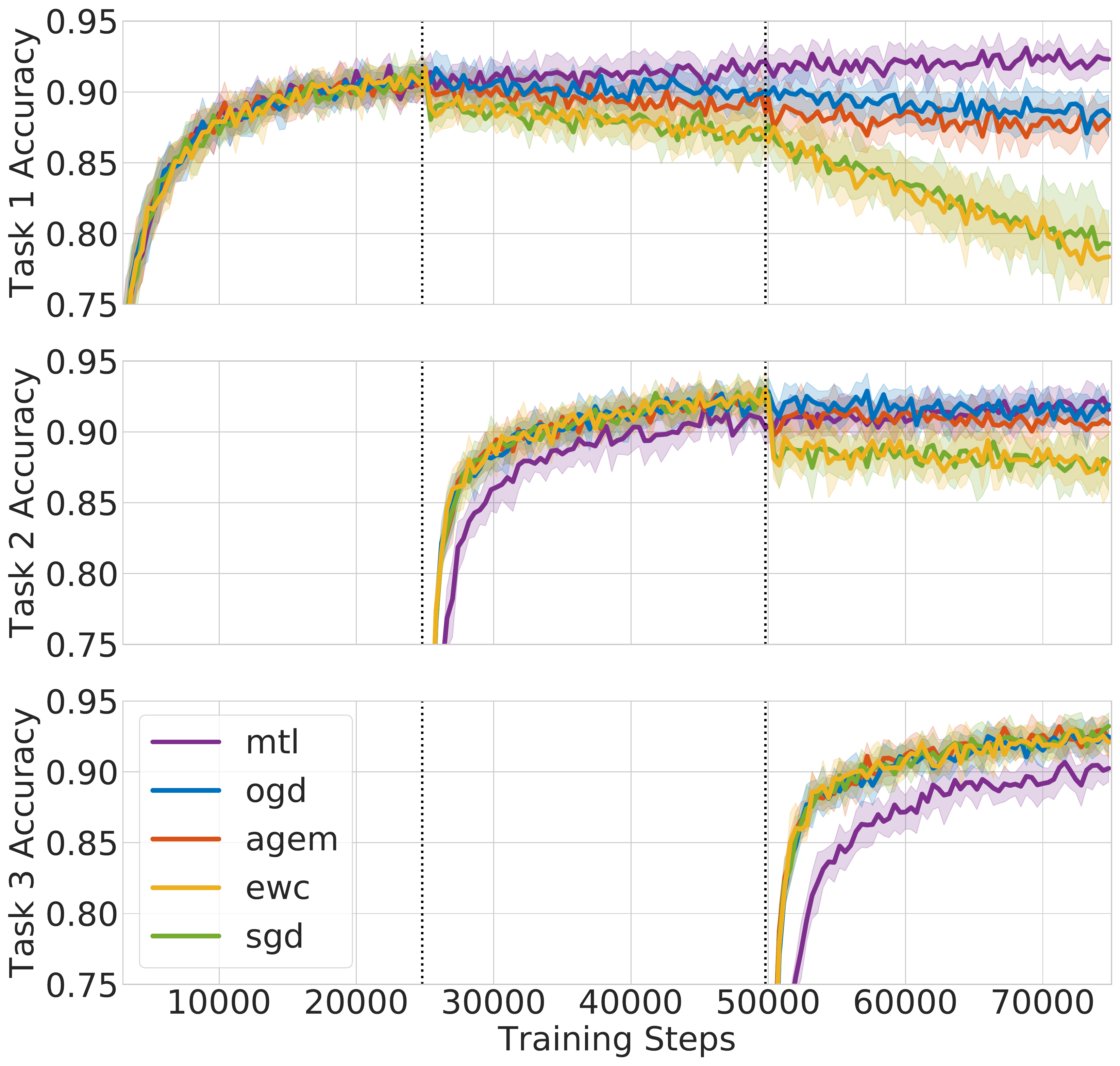} 
    \caption{\small Performance of different methods on permuted Mnist task. 3 different permutations ($p_1$, $p_2$, and $p_3$) are used and the model is trained to classify Mnist digits under permutation $p_1$ for 5 epochs, then under $p_2$ for 5 epochs and then under $p_3$ for 5 epochs. The vertical dashed lines represent the points in the training where the permutations switch. The top plot reports the accuracy of the model on batches of the Mnist test set under $p_1$; the middle plot, under $p_2$; and the bottom plot under $p_3$. The y-axis is truncated to show the details. Note that MTL represents a setting where the model is directly trained on all previous tasks. Because we keep constant batch size and number of epochs, the MTL method effectively sees one third of the task 3 data that other methods do. This is the reason that MTL learns slower on task 3 than other methods.}
    \label{fig:permuted_mnist_training}
\end{figure}

\subsection{Permuted Mnist}
\label{sec:permuted_mnist}
We tested our method on the Permuted Mnist setup described in~\citep{goodfellow2013empirical} and utilized in~\citep{kirkpatrick2017overcoming,chaudhry2018efficient} too. In this setup we generated a series of 3 permutations $p_1$, $p_2$, and $p_3$ that shuffle the pixels in an Mnist image. We designated task $T_i$ as the problem of classifying Mnist digits that have been shuffled under permutation $p_i$. We chose these permutations randomly so each task is equally hard and so the difference in accuracy between examples from task 1 and examples from task 3 after task 3 has been trained is a measure of how much the network is able to remember $p_1$.

Figure~\ref{fig:permuted_mnist_training} shows the accuracy of OGD and baselines on the Permuted Mnist task. The plot shows that OGD retains performance on task 1 examples as well as A-GEM even after training on task 3. Both methods perform slightly worse than a model that is able to train on all previous tasks (MTL), but significantly better than the naive sequential model (SGD) and than EWC.

\begin{table}[!tb]
\centering
\setlength\tabcolsep{3pt}
\resizebox{\linewidth}{!}{
\begin{tabular}{l|ccccc}
\toprule
	 & \multicolumn{5}{c}{\bf Accuracy $\pm$ Std. ($\%$)} \\
	 & \bf Task 1 & \bf Task 2 & \bf Task 3 & \bf Task 4 & \bf Task 5	\\
	\midrule
	\sc mtl & $93.2 \pm 1.3$ & $91.5 \pm 0.5$ & $91.3 \pm 0.7$ & $91.3 \pm 0.6$ & $88.4 \pm 0.8$\\
	\midrule
	\sc ogd & $79.5 \pm 2.3$ & ${\bf 88.9} \pm 0.7$ & ${\bf 89.6} \pm 0.3$ & ${\bf 91.8} \pm 0.9$ & $92.4 \pm 1.1$\\
	\bf \sc a-gem & ${\bf 85.5} \pm 1.7$ & $87.0 \pm 1.5$ & ${\bf 89.6} \pm 1.1$ & $91.2 \pm 0.8$ & ${\bf 93.9} \pm 1.0$\\
	\sc ewc & $64.5 \pm 2.9$ & $77.1 \pm 2.3$ & $80.4 \pm 2.1$ & $87.9 \pm 1.3$ & $93.0 \pm 0.5$\\
	\sc sgd & $60.6 \pm 4.3$ & $77.6 \pm 1.4$ & $79.9 \pm 2.1$ & $87.7 \pm 2.9$ & $92.4 \pm 1.1$\\
	\bottomrule
\end{tabular}
}
\label{tab:perm_mnist_accuracy_5task}
\caption{\small \textit{Permuted Mnist}: The accuracy of models  for test examples from the indicated class after being trained on all tasks in sequence, except the multi-task setup (\textsc{mtl}). The best continual learning results are highlighted in \textbf{bold}.}
\end{table}

We extended this experiments to $5$ permutations and tasks in the same manner. For this experiment, we evaluated the classifier after training had completed (at the end of task 5) and measured the accuracy for examples from each of task 1\ldots 5. Table~\ref{tab:perm_mnist_accuracy_5task} reports these accuracies for OGD and the baseline training methods. The results suggest that the overall performance of OGD is significantly better than EWC and SGD while being on par with A-GEM. 

\subsection{Rotated Mnist}
\label{sec:rotated_mnist}

We further evaluated our approach on identifying rotated Mnist digits. The training setup is similar to Permuted Mnist except that instead of arbitrary permutation, we used fixed rotations of the Mnist digits. Here we started with a two task problem: task 1 is to classify standard Mnist digits and then task 2 is to classify those digits rotated by a fixed angle. 

Figure~\ref{fig:rotated_mnist} shows the accuracy of the model when classifying task 1 examples (normal, un-rotated Mnist digits) after the end of training on task 2 (rotated Mnist digits). We report this as a function of the angle of rotation. One can see that, as the angle of rotation increases, the task becomes harder. Even in this harder setting, we still observe that OGD and A-GEM exhibit similar levels of performance.

\begin{figure}[!tb]
 \centering
 \includegraphics[width=0.49\textwidth]{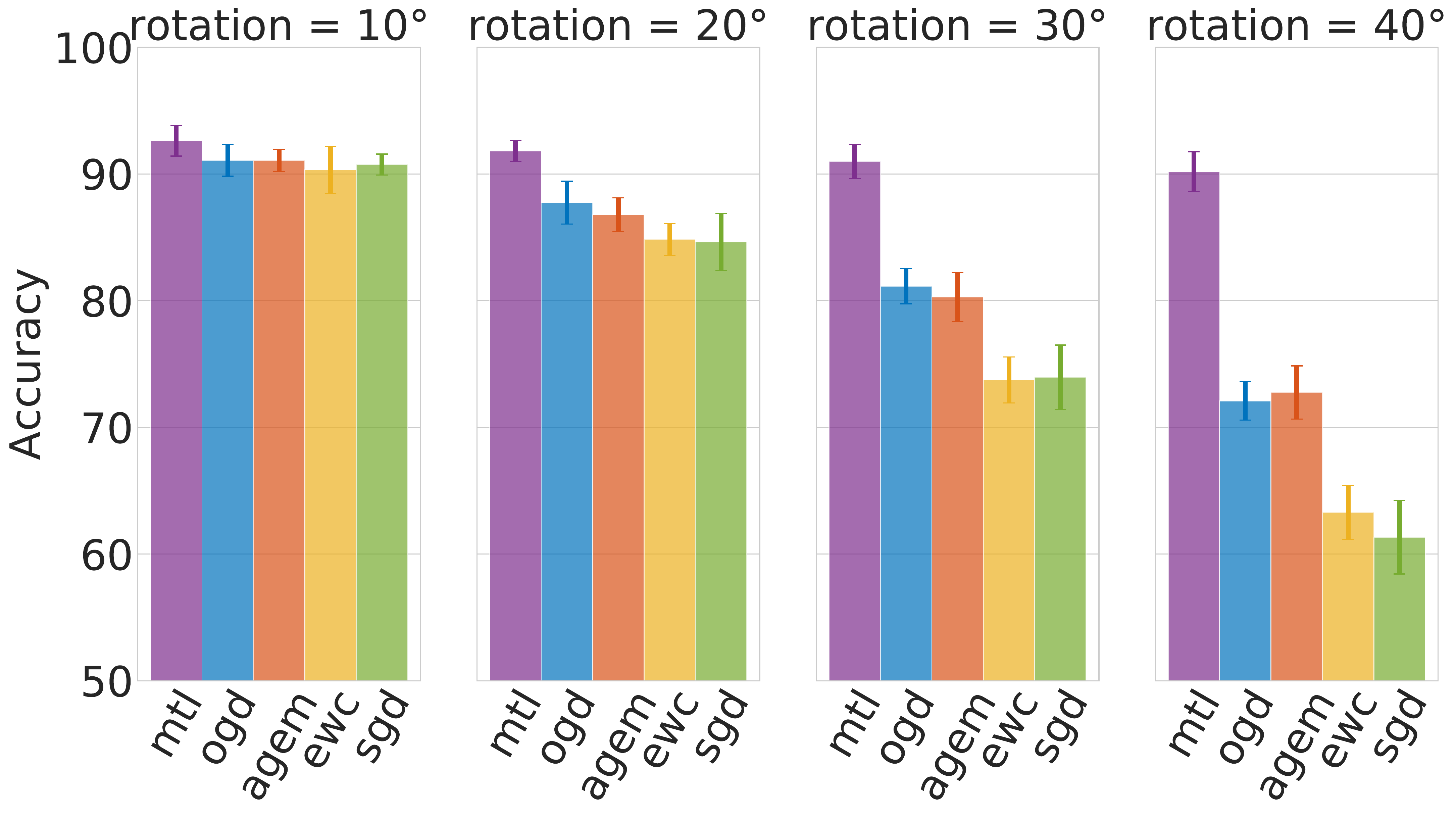}
 \caption{\small \textit{Rotated Mnist}: Accuracies of multiple continual learning methods. Every classifier is trained for 5 epochs on standard Mnist and then trained for another 5 epochs on a variant of Mnist whose images are rotated by the specified angle. The accuracy is computed over the entire original (un-rotated) Mnist test set after the model being trained on the rotated dataset. Each bar represents the mean accuracy over 10 independent runs and the error bars reflect their standard deviations. MTL represents the (non-continual) multi-task learning setting where the model is trained with the combined data from all previous tasks.}
 \label{fig:rotated_mnist}
\end{figure}

In the same way as the previous experiment, we extended the rotated Mnist experiment to more tasks by training a classifier on 5 rotated Mnist tasks with increasing angle of rotation. We defined the tasks as classification under angles of $T_1 = Rot(0^\circ)$, $T_2 = Rot(10^\circ)$, \ldots, $T_5 = Rot(40^\circ)$, and train the models in that order. Table~\ref{tab:rot_mnist_accuracy_5task} shows the accuracy of the fully-trained model at classifying examples from each tasks. We can observe that OGD outperforms other methods on 10, 20, and 30 degree rotations.

\begin{table}[!tb]
\centering
\setlength\tabcolsep{3pt}
\resizebox{\linewidth}{!}{
\begin{tabular}{lccccc}
\toprule
	 & \multicolumn{5}{c}{\bf Accuracy $\pm$ Std. ($\%$)} \\
	 & \bf Task 1 & \bf Task 2 & \bf Task 3 & \bf Task 4 & \bf Task 5	\\
	\midrule
	\sc mtl & $92.1 \pm 0.9$ & $94.3 \pm 0.9$ & $95.2 \pm 0.9$ & $93.4 \pm 1.1$ & $90.5 \pm 1.5$\\
	\midrule
	\sc ogd & ${\bf 75.6} \pm 2.1$ & ${\bf 86.6} \pm 1.3$ & ${\bf 91.7} \pm 1.1$ & $94.3 \pm 0.8$ & $93.4 \pm 1.1$\\
    \sc a-gem & $72.6 \pm 1.8$ & $84.4 \pm 1.6$ & $91.0 \pm 1.1$ & $93.9 \pm 0.6$ & ${\bf 94.6} \pm 1.1$\\
	\sc ewc & $61.9 \pm 2.0$ & $78.1 \pm 1.8$ & $89.0 \pm 1.6$ & $94.4 \pm 0.7$ & $93.9 \pm 0.6$\\
	\sc sgd & $62.9 \pm 1.0$ & $76.5 \pm 1.5$ & $88.6 \pm 1.4$ & ${\bf 95.1} \pm 0.5$ & $94.1 \pm 1.1$\\
	\bottomrule
\end{tabular}
}
\caption{\small \textit{Rotated Mnist}: The accuracy of models  for test examples from the indicated class after being trained on all tasks in sequence, except the multi-task setup (\textsc{mtl}). The best continual learning results are highlighted in \textbf{bold}.}
\label{tab:rot_mnist_accuracy_5task}
\end{table}

\subsection{Split Mnist}
\label{sec:split_mnist}

We also tested OGD in a setting where the labels between task 1 and task 2 are disjoint. We followed the setup for split Mnist laid out in~\cite{zenke2017continual} with some variations.  We defined a set of tasks $T_1$ \ldots $T_N$, with task $T_i$ defined by a series of integers $t_i^1\ldots t_i^{k_i}$ with $0\leq t_i^j \leq 10$ and $t_i^j = t_{i'}^{j'}$ if and only if $i=i'$ and $j=j'$. For each task $T_i$, then, the task is to classify Mnist digits with labels in $\{t_i^j\}$.

Because a given task does not contain all labels, we used a slightly different architecture for this task compared to other tasks. Instead of having a single output layer containing 10-logits for all the Mnist classes, we used separate heads for each task, where each head has the same number of logits as there are classes in the associated task.  This means that, for each task $T_i$, the softmax and cross-entropy calculation only runs over the logits and labels $\{t_i^j\}$. We found that this model has higher performance under all methods than a model using a joint head.

We began with a two task classification problem, where the Mnist dataset is split into two disjoint sets each containing 5 labels. The tasks are then just to classify examples from the set associated with each task. Figure~\ref{fig:split_mnist_2class} shows the accuracy of the fully-trained model to classify images from task $T_1$. We report the results for 5 different partitions of the labels into the task sets, to ensure that the partition does not have a strong effect on the results. In all cases, we observe the OGD performs the best, beating A-GEM again. We also observe that the performance order is preserved across different configurations of the experiment.
\begin{figure}[!tb]
 \centering
 \includegraphics[width=0.49\textwidth]{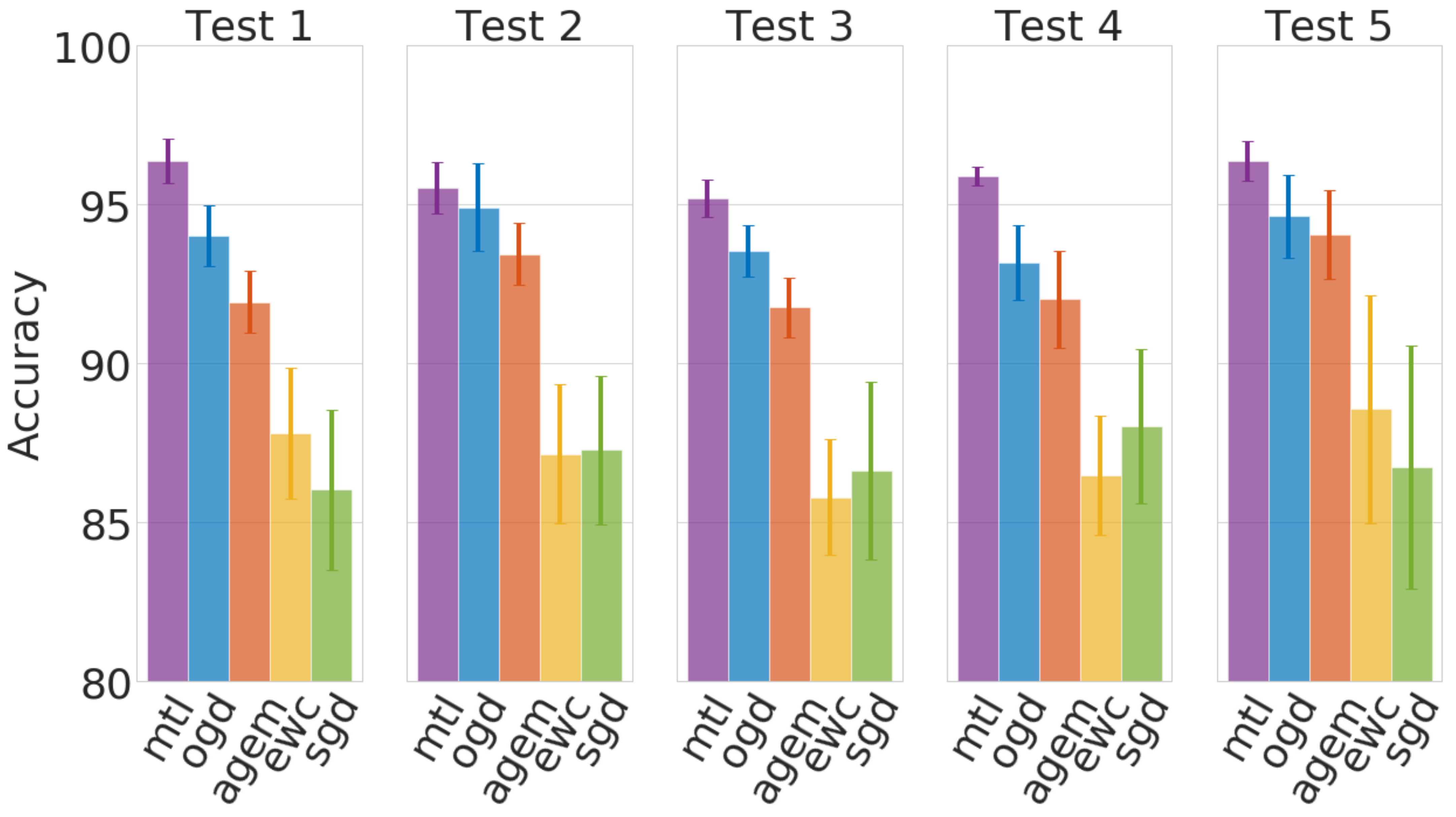}
 \caption{\small \textit{Split Mnist}: Accuracies of multiple continual learning methods. The training regime is the same as that of Figure~\ref{fig:permuted_mnist_training}. The reported value is the accuracy on task 1 after the model being trained on task 2. Different plots correspond to different configurations, \textit{i.e.}, different partitions of the Mnist labels into task 1 and task 2.}
 \label{fig:split_mnist_2class}
\end{figure}

\begin{table}[!tb]
\centering
\setlength\tabcolsep{3pt}
\resizebox{\linewidth}{!}{
\begin{tabular}{lccccc}
\toprule
	 & \multicolumn{5}{c}{\bf Accuracy $\pm$ Std. ($\%$)} \\
	 & \bf Task 1 & \bf Task 2 & \bf Task 3 & \bf Task 4 & \bf Task 5	\\
\midrule
	\sc mtl & $99.6 \pm 0.2 $ & $99.8 \pm 0.1 $ & $98.8 \pm 0.2 $ & $98.2 \pm 0.4 $ & $99.1 \pm 0.2 $\\
	\midrule
	\sc ogd & ${\bf 98.6} \pm 0.8 $ & ${\bf 99.5} \pm 0.1 $ & ${\bf 98.0} \pm 0.5 $ & ${\bf 98.8} \pm 0.5 $ & $99.2 \pm 0.3 $\\
	\sc a-gem & $92.9 \pm 2.6 $ & $96.3 \pm 2.1 $ & $86.5 \pm 1.6 $ & $92.3 \pm 2.3 $ & $99.3 \pm 0.2 $\\
	\sc ewc & $90.2 \pm 5.7 $ & $98.9 \pm 0.2 $ & $91.1 \pm 3.5 $ & $94.4 \pm 2.0 $ & $99.3 \pm 0.2$\\
	\sc sgd & $88.2 \pm 5.9$ & $98.4 \pm 0.9 $ & $90.3 \pm 4.5 $ & $95.2 \pm 1.0 $ & ${\bf 99.4} \pm 0.2 $\\
\bottomrule
\end{tabular}
}
\caption{\small \textit{Split Mnist}\footnotemark : The accuracy of models  for test examples from the indicated class after being trained on all tasks in sequence, except the multi-task setup (\textsc{mtl}). The best continual learning results are highlighted in \textbf{bold}.}
\label{tab:split_mnist_accuracy_5task}
\end{table}
\footnotetext{The accuracy for different assignments of labels to tasks in Table \ref{tab:split_mnist_accuracy_5task} can be found in Appendix~\ref{sec:app_split_mnist}.}

We again generalized this experiment to a longer sequence of tasks by splitting Mnist into 5 tasks, each with two classes. We used a multi-headed architecture as in the 2 task case. We report the accuracy of the fully trained model on examples from each of the 5 classes in  Table~\ref{tab:split_mnist_accuracy_5task}. As in the previous case, we evaluated this on multiple partitions of the labels; the results from other partitions are shown in the appendix. In this setting, OGD performs very closely to the multi-task training benchmark and consistently outperforms the other baselines.

\section{Related Work}
\vspace{-4mm}
\label{sec:related_work}
There is a growing interest in measuring catastrophic forgetting~\citep{toneva2018empirical,kemker2018measuring}, evaluating continual learning algorithms~\citep{farquhar2018towards,hayes2018new,diaz2018don,de2019continual,hsu2018re}, and understating this phenomenon~\citep{nguyen2019toward,farquhar2019unifying}.
The existing work on alleviating catastrophic forgetting can be divided into a few categories.

The \emph{expansion} based methods allocate new neurons or layers or modules to accommodate new tasks while utilizing the shared representation learned from previous ones.   \citet{rusu2016progressive} proposed progressive neural networks in which parameters for the original task are untouched while the architecture is expanded by allocating new sub-networks with fixed capacity to be trained on the new task. 
Similarly, \citet{xiao2014error} proposed a method in which the network not only grows in capacity, but forms a hierarchical structure as new tasks arrive at the model. 
\citet{yoon2018lifelong} proposed a dynamically expanding network that either retrain or expand the network capacity upon arrival of a new task with only the necessary number of units  by splitting/duplicating units and timestamping them.
\citet{draelos2017neurogenesis} used auto-encoder to dynamically decide to add neurons for samples with high loss and whether the older data needs to be retrained or not. 
Along this idea ~\citet{Jerfel2018ReconcilingMA} proposed to use  Dirichlet process mixture of hierarchical
Bayesian models over the parameters of
neural networks to dynamically cope with the new tasks.
Recently, \cite{li2019learn} proposed to utilize the neural architecture search to find the optimal structure for each of the sequential tasks.
These methods avoid storing data and are aligned with neurogenesis in the brain~\cite{aimone2009computational} but may be complex for the current neural network libraries.

In the \emph{regularization} based approaches, catastrophic forgetting is tackled by imposing constraints on the weight updates of the neural network according to some importance measure for previous tasks. The difference lies in the way how importance weights are computed. In Elastic Weight Consolidation (EWC)~\citep{kirkpatrick2017overcoming} the importance weights are the diagonal values of  the Fisher information matrix which approximates the posterior distribution of the weights. Along this Bayesian perspective, \citet{titsias2019functional} proposed to work over the function space rather than the parameters of a deep neural network to avoid forgetting a previous task by constructing and memorizing an approximate posterior belief over the underlying task-specific function. Employing other Bayesian techniques, \citet{nguyen2017variational}  combine online variational inference and recent advances in Monte Carlo methods. \citet{ritter2018online} recursively approximate the posterior after
every task with a Gaussian Laplace approximation of the Hessian for continual learning, and~\citet{ebrahimi2019uncertainty} used uncertainty measures to help continual learning.
\citet{schwarz2018progress} proposed a cycle of active learning (progression) followed by consolidation (compression) that requires no architecture growth and no access to or storing of previous data or tasks which is similar in spirit to distillation based methods of continual learning~\cite{li2017learning,hu2018overcoming}.
Knowledge Distillation~\citep{hinton2015distilling} and its many variants\citep{romero2014fitnets,mirzadeh2019improved} are useful to retain the previous information.

In \citep{zenke2017continual} each parameter accumulates task relevant information over time, and exploits this information to rapidly store new tasks without forgetting old ones. 
\citet{lee2017overcoming} incrementally match the moment of the posterior distribution of the neural network trained on the first and the second task to regularize its update on the latter. 
Other works along this line are \citep{aljundi2018memory,kolouri2019attention} which penalize the weights based on a Hebbian like update rule.
These approaches are well motivated by neuro-biological models of memory~\cite{fusi2005cascade,kaplanis2018continual} and are computationally fast and do not require storing data. However, these \emph{consolidated} weights reduce the degree of freedom of the neural network. In other words, they decrease the effective volume of parameter space to search for a configuration that can satisfy both the old and new tasks.

The \emph{repetition} based methods employ memory systems that store previous data or, alternatively, train a generative model for the first task and replay them interleaved with samples drawn from the new task.
\citet{shin2017continual,kamra2017deep,zhang2019prototype,rios2018closed} learned a generative
model to capture the data distribution of previous tasks, along with the current task's data to train the new model so that the forgetting can be alleviated.
\citet{luders2016continual} used a Neural Turing Machine that enables agents to store long-term memories by progressively employing additional memory components. In the context of Reinforcement learning~\citet{rolnick2018experience} utilized on-policy learning on fresh experiences to adapt rapidly to new tasks, while using off-policy learning with behavioral cloning on replay experience to maintain and modestly enhance performance on past tasks.
\citet{lopez2017gradient} proposed Gradient Episodic Memory (GEM) to efficiently use an episodic storage by following loss gradients on incoming task to the maximum extent while altering them so that they do not interfere with past memories. While minimizing the loss on the current task GEM treats the losses on  the episodic
memories of previous tasks as inequality constraints, avoiding their increase but allowing their decrease. \citet{chaudhry2018efficient} improved GEM by changing the loss function and proposed dubbed Averaged GEM (A-GEM), which enjoys the same or even better performance.
\citet{riemer2018learning} combined experience replay with optimization based meta-learning to enforce gradient alignment across examples in order to learn parameters that make interference based on future gradients less likely. 
A few other works have utilized gradient information to protect previous knowledge.
\citet{he2018overcoming} proposed a variant of the back-propagation algorithm named conceptor-aided backprop that shields gradients against degradation of previously learned tasks. 
\citet{zeng2018continuous} ensure that gradient updates occur only in the orthogonal directions to the input of previous tasks. 
This class of methods also have their root in neuroscience~\citep{kumaran2012generalization} making training samples as identically distributed as possible. However, they need to store a portion of the data or learning a generative model upon them, which might not be possible in some settings, \textit{e.g.}, with user data when privacy matters. Moreover, many of these methods work with the gradients of the loss, which can be close to zero for many samples and therefore convey less information on previous tasks. In contrast, we work with the gradients of the model (logits or predictions) which is the actual knowledge to be preserved on the course of continual learning.
By providing more effective shield of gradients through projecting to the space of previous model gradients, we achieve better protection to previously acquired knowledge, yielding highly competitive results in empirical tests compared to others.

Continual Learning as a sub-field in AI has close connection and ties to other recent efforts in machine learning. 
\emph{Meta learning} algorithms use a data-driven inductive bias to enhance learning new tasks~\citep{Jerfel2018ReconcilingMA,he2018overcoming,vuorio2018meta,al2017continuous,riemer2018learning}. \emph{Few shot learning} also serves the same purpose and can be leveraged in continual learning~\citep{wen2018few,gidaris2018dynamic} and vice versa. The way we treat previous knowledge (i.e. through the model prediction gradients not the actual data) is also related \emph{differential private learning}~\citep{wu2017bolt,li2018differentially,pihur2018differentially,han2018privacy} and \emph{federated learning}~\citep{bonawitz2019towards,smith2017federated,vepakomma2018split}. \emph{Multi-task learning}~\citep{sorokin2019continual}, \emph{curriculum learning}~\citep{bengio2009curriculum}, and \emph{transfer learning}~\citep{pan2009survey,li2019cross} are other related areas helpful to develop better continual learning machines that do not catastrophically forget previous experiences.

\vspace{-3mm}
\section{Conclusion and Outlook}
\vspace{-4mm}

In this paper, we propose to project the current gradient steps to the orthogonal space of neural network predictions on previous data points. The goal is to minimally interfere with the already learned knowledge while gradually stepping towards learning new tasks. We have demonstrated that our method matches or exceeds other state-of-the-art methods on a variety of benchmark experiments. We have observed that OGD is able to retain information over many tasks and achieved particularly strong results on the split Mnist benchmark.

There are several avenues for future study based on this technique. Firstly, because we cannot store gradients for the full datasets there is some forgetting happening. Finding a way to store more gradients or prioritize the important directions would improve the performance. One can also maintain higher-order derivatives of the model for a more accurate representation of previously learned knowledge, at the expense of more memory and computation.
Secondly, we have observed that all methods (including ours) fail considerably when the tasks are dissimilar (for example rotations larger than 90 degrees for the Mnist task). This calls for a lot more future research to be invested in this important yet under-explored problem of continual learning. 
Thirdly, it is observed that our method is sensitive to the learning rate and it sometimes fail to produce comparable results to A-GEM for large learning rates. It's aligned with our expectation that the learning rate is determining the locality and the neighborhood of the search. The gradients of the model predictions at optimal points are a good approximation for the gradients on others if they lie in a close neighborhood. Further work on coping with this would allows OGD to apply to settings where higher learning rates are desired. Another interesting direction for future research is to extend this idea to other types of optimizers such as Adam or Adagrad.

Finally, it is worth noting that the implications of the proposed Orothogonal Gradient Descent goes beyond the standard continual learning setup we described. Firstly, it does not require tasks to be identified and distinguished. OGD minimally interfere with previously seen data points no matter what class they belong to. This makes it applicable when the task shift does not arrive as a distinct event, but rather a gradual shift~\citep{he2019task}. Moreover, it is applicable to standard learning paradigm where one does not have the luxury of iterating over numerous epochs, as it can preserve information that has not yet been strongly encoded in the weights of the network. A principled extension and verification on common gradient neural network training methods is left as future work.

\subsubsection*{Acknowledgements}
The authors would like to thank Dilan Gorur,  Jonathan Schwarz, Jiachen Yang, and Yee Whye Teh for the comments and discussions.

\bibliographystyle{apalike}
\bibliography{main}

\begin{appendix}
\onecolumn

\section{Additional Experiment Information}

\subsection{Variants of Orthogonal Gradient Descent}
\label{sec:app_variants}

As described in Section~\ref{sec:proposed_ogd} we test OGD with three settings for the which gradients to store, and 3 settings for how many gradients to store. OGD-ALL stores the gradients with respect to all logits of the model. OGD-AVG stores the gradients with respect to the average of all logits. OGD-GTL stores the gradient with respect to the ground truth logit. We run tests storing 20, 200, and 2000 gradients. Table~\ref{tab:ogd_grad_methods} summarizes the results of this experiment. We observe that increasing the number of gradients improves performance across the board (which is expected). We observe the OGD-GTL and OGD-ALL have similar performance in most cases, with a bit of an edge to OGD-GTL. OGD-AVG performs worse in most cases.

\begin{table*}[h]
\centering
\begin{tabular}{l|ccc}
\toprule
 rotated & \multicolumn{3}{c}{Accuracy $\pm$ Std (\%)} \\
Mnist & 20 Grads & 200 Grads & 2000 Grads \\
\midrule
all & 75.7 $\pm$ 2.6 & 79.9 $\pm$ 1.4 & 86.6 $\pm$ 1.0 \\
average & 75.3 $\pm$ 2.4 & 75.5 $\pm$ 1.4 & 77.7 $\pm$ 1.6 \\
ground truth & 76.4 $\pm$ 2.2 & 82.9 $\pm$ 1.6 & 87.1 $\pm$ 1.1 \\
\bottomrule

\toprule
 permuted & \multicolumn{3}{c}{Accuracy $\pm$ Std (\%)} \\
Mnist & 20 Grads & 200 Grads & 2000 Grads \\
\midrule
all & 87.3 $\pm$ 2.8 & 89.7 $\pm$ 1.5 & 90.5 $\pm$ 0.9 \\
average & 86.8 $\pm$ 1.4 & 86.9 $\pm$ 1.4 & 89.4 $\pm$ 1.7 \\
ground truth & 86.5 $\pm$ 1.5 & 89.4 $\pm$ 1.0 & 91.4 $\pm$ 1.7 \\
\bottomrule
\end{tabular}
\label{tab:ogd_grad_methods}
\caption{The performance of various OGD gradient methods as a function of number of gradients stored on rotated Mnist (top) and permuted Mnist (bottom).  Numbers are the accuracy on task 1 after fully training on task 2.}
\end{table*}

\subsection{Increased Training Epochs}

We  study the effect that increasing the number of training epochs has on the performance of the different training methods on permuted Mnist. For the Mnist experiments in the Section~\ref{sec:Experiments}, we train for 5 epochs per task, which is enough to achieve $93\%$ accuracy on vanilla Mnist classification and is in the regime short enough to avoid over-fitting. In order to determine whether increased training time has an effect on the performance in the multi-task setting, we train a classifier on 2-task permuted Mnist running each task training for 20, 40, 80, and 120 epochs and report the classification accuracy on task 1 after task 2 has finished. The results are shown in  Figure~\ref{fig:epoch_sweep}.  
Note that A-GEM and OGD have maintained competitive performance with increasing number of epochs while in the case of SGD and EWC the performance first increases and then drops.

\begin{figure}[h]
 \centering
 \includegraphics[width=0.48\textwidth]{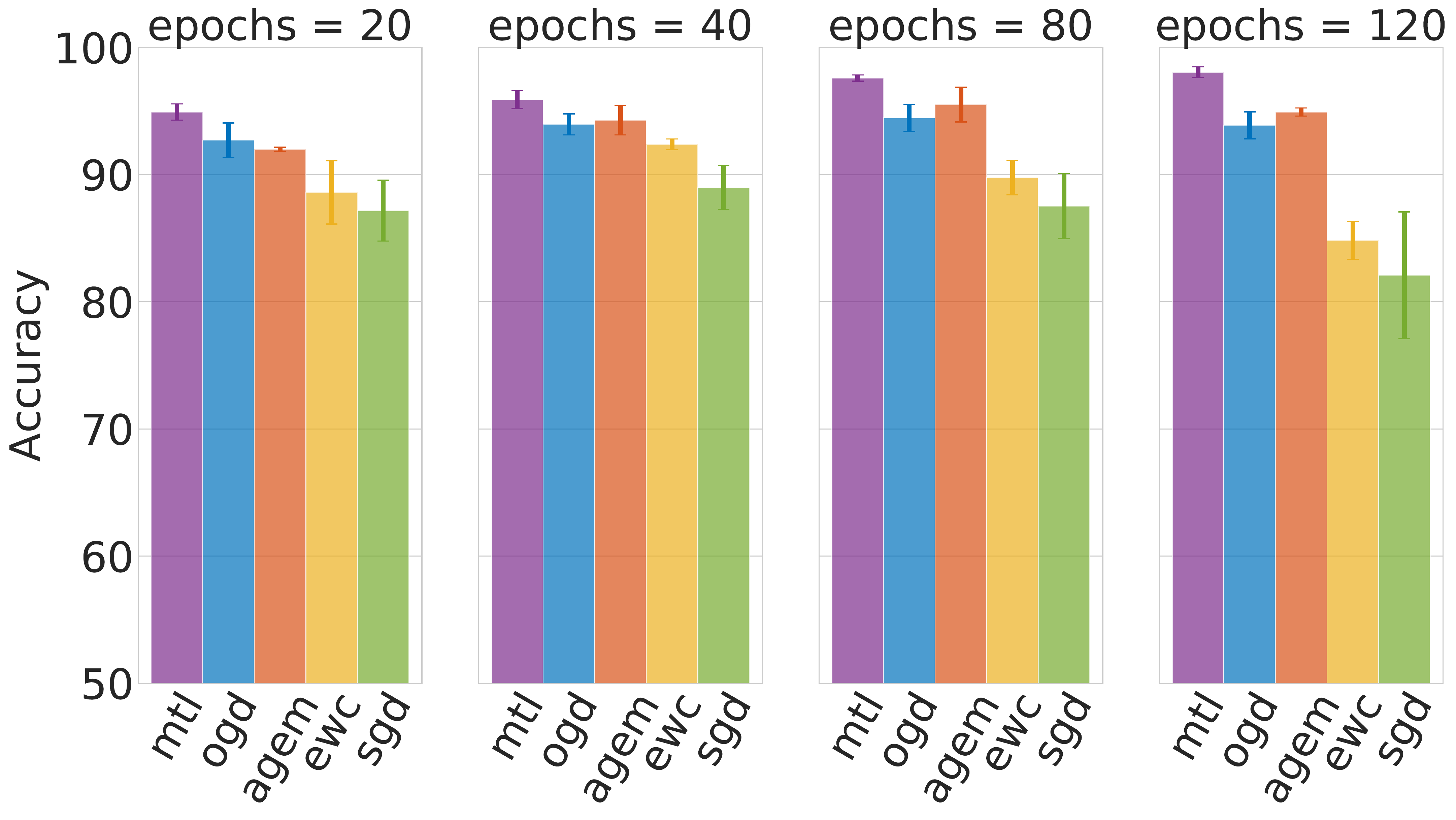}
 \caption{The performance of OGD versus others as a function of the number of training epochs for each task on permuted Mnist. 
 }
 \label{fig:epoch_sweep}
\end{figure}

\subsection{Split Mnist}
\label{sec:app_split_mnist}

We present the results of the split Mnist study described in Section~\ref{sec:split_mnist} on 3 other instances of the split Mnist task. These instances differ by the way the Mnist classes are split into tasks and the order in which the tasks are presented. Tables~\ref{tab:split_mnist_accuracy_5task_test2},~\ref{tab:split_mnist_accuracy_5task_test3}, and~\ref{tab:split_mnist_accuracy_5task_test4} show the results on these tests.  We can see that the ordering of the task methods is preserved in all tests: MTL and OGD are very close in performance, with a gap before A-GEM, and finally EWC and SGD.

\begin{table*}[h]
\centering
\begin{tabular}{c|ccccc}
\toprule
	 & \multicolumn{5}{c}{Accuracy $\pm$ Std. (\%)} \\
	 & Task 1 & Task 2 & Task 3 & Task 4 & Task 5	\\
	\midrule
	mtl & $99.6 \pm 0.2$ & $98.5 \pm 0.4$ & $97.7 \pm 0.4$ & $96.8 \pm 1.0$ & $98.7 \pm 0.3$\\
	ogd & $99.6 \pm 0.4$ & $97.7 \pm 0.1$ & $97.3 \pm 0.5$ & $98.0 \pm 0.9$ & $99.3 \pm 0.1$\\
	agem & $99.2 \pm 0.6$ & $91.4 \pm 3.7$ & $91.4 \pm 0.9$ & $87.1 \pm 3.9$ & $98.9 \pm 0.3$\\
	ewc & $97.0 \pm 3.2$ & $92.7 \pm 3.8$ & $91.9 \pm 5.7$ & $94.3 \pm 2.2$ & $99.2 \pm 0.6$\\
	sgd & $97.4 \pm 2.4$ & $92.2 \pm 3.5$ & $89.2 \pm 8.6$ & $94.5 \pm 1.4$ & $99.1 \pm 0.3$\\
\bottomrule
\end{tabular}
\caption{The accuracy of models trained by different methods on split Mnist. The reported values are the accuracy of the model for test examples from the indicated class after the model has been trained on all tasks in sequence. This table contains the same settings as Table~\ref{tab:split_mnist_accuracy_5task}, but with a different order of Mnist classes assigned to the tasks.}
\label{tab:split_mnist_accuracy_5task_test2}
\end{table*}

\begin{table*}[h]
\centering
\begin{tabular}{c|ccccc}
\toprule
	 & \multicolumn{5}{c}{Accuracy $\pm$ Std. (\%)} \\
	 & Task 1 & Task 2 & Task 3 & Task 4 & Task 5	\\
	\midrule
	mtl & $99.4 \pm 0.2$ & $99.2 \pm 0.3$ & $98.6 \pm 0.4$ & $99.7 \pm 0.3$ & $98.6 \pm 0.5$\\
	ogd & $99.0 \pm 0.4$ & $98.6 \pm 0.1$ & $98.0 \pm 0.2$ & $99.6 \pm 0.3$ & $99.6 \pm 0.2$\\
	agem & $94.1 \pm 2.9$ & $93.8 \pm 5.5$ & $90.6 \pm 2.2$ & $99.4 \pm 0.3$ & $99.4 \pm 0.3$\\
	ewc & $94.8 \pm 2.9$ & $95.3 \pm 3.1$ & $95.5 \pm 0.6$ & $99.3 \pm 0.2$ & $99.3 \pm 0.2$\\
	sgd & $94.6 \pm 2.1$ & $96.3 \pm 1.2$ & $95.0 \pm 1.6$ & $99.3 \pm 0.4$ & $99.3 \pm 0.2$\\
\bottomrule
\end{tabular}
\caption{The accuracy of models trained by different methods on split Mnist. The reported values are the accuracy of the model for test examples from the indicated class after the model has been trained on all tasks in sequence.  This table contains the same settings as Table~\ref{tab:split_mnist_accuracy_5task}, but with a different order of Mnist classes assigned to the tasks. }
\label{tab:split_mnist_accuracy_5task_test3}
\end{table*}

\begin{table*}[h]
\centering
\begin{tabular}{c|ccccc}
\toprule
	 & \multicolumn{5}{c}{Accuracy $\pm$ Std. (\%)} \\
	 & Task 1 & Task 2 & Task 3 & Task 4 & Task 5	\\
	\midrule
	mtl & $98.4 \pm 0.2$ & $100.0 \pm 0.0$ & $98.6 \pm 0.3$ & $99.5 \pm 0.2$ & $98.9 \pm 0.5$\\
	ogd & $98.1 \pm 0.8$ & $99.9 \pm 0.1$ & $97.8 \pm 0.6$ & $99.4 \pm 0.3$ & $99.5 \pm 0.3$\\
	agem & $92.1 \pm 2.7$ & $93.8 \pm 8.2$ & $93.0 \pm 3.5$ & $98.6 \pm 0.5$ & $99.5 \pm 0.3$\\
	ewc & $92.5 \pm 2.2$ & $98.1 \pm 3.0$ & $94.0 \pm 0.9$ & $99.4 \pm 0.2$ & $99.5 \pm 0.3$\\
	sgd & $89.6 \pm 4.4$ & $98.9 \pm 1.0$ & $89.1 \pm 7.9$ & $98.9 \pm 0.7$ & $99.5 \pm 0.3$\\
\bottomrule
\end{tabular}
\caption{The accuracy of models trained by different methods on split Mnist. The reported values are the accuracy of the model for test examples from the indicated class after the model has been trained on all tasks in sequence.  This table contains the same settings as Table~\ref{tab:split_mnist_accuracy_5task}, but with a different order of Mnist classes assigned to the tasks. }
\label{tab:split_mnist_accuracy_5task_test4}
\end{table*}

\end{appendix}

\end{document}